\documentclass[conference]{IEEEtran}

\usepackage{cite}
\usepackage{amsmath,amssymb,amsfonts}
\usepackage{algorithmic}
\usepackage{graphicx}
\usepackage{textcomp}
\usepackage{xcolor}
\usepackage{booktabs}
\usepackage{multirow}
\usepackage{url}
\usepackage{hyperref}
\usepackage{amsthm}
\usepackage{tikz}

\newtheorem{proposition}{Proposition}
\newtheorem{theorem}{Theorem}
\newtheorem{definition}{Definition}
\newtheorem{property}{Property}

\def\BibTeX{{\rm B\kern-.05em{\sc i\kern-.025em b}\kern-.08em
    T\kern-.1667em\lower.7ex\hbox{E}\kern-.125emX}}

\begin{document}

\title{Chronofy: A Temporal-Logical Decay Architecture for\\Information Validity in Time-Aware Retrieval-Augmented Generation}

\author{
\IEEEauthorblockN{Muntaser Syed\textsuperscript{1}, Marius Silaghi\textsuperscript{1}, Sheikh Abujar\textsuperscript{2}, Sharun Akter\textsuperscript{3}}
\IEEEauthorblockA{\textsuperscript{1}Florida Institute of Technology, \textsuperscript{2}University of Alabama at Birmingham, \textsuperscript{3}Daffodil International University\\
msyed2011@my.fit.edu, msilaghi@fit.edu, sabujar@uab.edu, sharun.cse@diu.edu.bd}
}

\maketitle

\begin{abstract}
Retrieval-Augmented Generation (RAG) systems retrieve and integrate external knowledge to ground large language model (LLM) outputs. However, current RAG architectures treat all retrieved facts as equally valid regardless of temporal provenance, leading to \emph{temporal hallucination}, where plausible but obsolete facts corrupt the output. We present \textbf{Chronofy}, a three-layer neuro-symbolic framework implementing the Temporal-Logical Decay Architecture (TLDA) that embeds temporal validity directly into the representation, retrieval, and reasoning layers of RAG systems. Layer~1 reserves a dedicated temporal subspace within Matryoshka embeddings to make fact age structurally irremovable from the representation. Layer~2 integrates learnable exponential decay functions into graph-based retrieval, where the decay coefficient $\beta_j$ is grounded in Bayesian decision theory as an approximation of twice the latent process mean-reversion rate. Layer~3 applies Signal Temporal Logic (STL) robustness functions to evaluate the \emph{temporal validity of retrieved knowledge}, not LLM output confidence, and enforces the possibilistic weakest-link principle to bound output confidence by the most decayed evidence in the reasoning chain. We evaluate Chronofy on temporal knowledge graph forecasting benchmarks, the TimE temporal QA benchmark, and a domain-specific sensitivity analysis, demonstrating that explicit temporal decay modeling improves retrieval precision, reduces temporal hallucination, and enables principled data re-acquisition triggers when temporal context is insufficient.
\end{abstract}


\section{Introduction}

Consider a clinical decision support system queried about a patient's risk for cardiac arrhythmia. The system retrieves two serum potassium readings: one from yesterday (4.1~mEq/L) and one from six months ago (3.2~mEq/L). A standard RAG pipeline, retrieving by semantic similarity alone, may weight both equally, or even prefer the six-month-old reading if it is more semantically aligned with the query.

This failure mode is not confined to clinical settings. In financial analytics, market data decays within hours; in legal compliance, regulatory interpretations are superseded by new rulings. Recent studies have identified a 17\% ``phantom rate'' in AI-assisted citations where semantically correct concepts are paired with hallucinated temporal metadata~\cite{ilter2026gap}.

Existing approaches address fragments of this problem. TempValid~\cite{huang2024tempvalid} introduces exponential decay for temporal knowledge graph forecasting but operates only at the knowledge graph layer without RAG integration or formal verification. STAR-RAG~\cite{zhu2025starrag} provides time-aligned graph summarization for retrieval but uses structural time costs rather than explicit decay functions. Mao et al.~\cite{mao2025stl} apply Signal Temporal Logic robustness to LLM chain-of-thought reasoning, but evaluate \emph{output confidence trajectories}, not the \emph{temporal validity of the retrieved knowledge} feeding the reasoning chain.

To our knowledge, no prior work combines temporal decay mathematics, time-aware retrieval, and formal temporal logic verification in a single end-to-end RAG system.

We present \textbf{Chronofy}, implementing the Temporal-Logical Decay Architecture (TLDA), a three-layer neuro-symbolic framework that addresses this gap. Our contributions are:

\begin{enumerate}
    \item A three-layer architecture integrating temporal subspace embeddings (Layer~1), decay-weighted graph retrieval (Layer~2), and STL-based knowledge validity verification (Layer~3) into a unified RAG pipeline.
    \item Application of Signal Temporal Logic robustness to \emph{knowledge temporal validity} rather than LLM confidence calibration, with a formal proof that the weakest-link principle bounds output confidence by the most decayed evidence in the reasoning chain.
    \item A decision-theoretic grounding of $\beta_j$ as twice the latent process mean-reversion rate under Gaussian dynamics (Proposition~\ref{prop:beta}), replacing ad hoc penalties with a principled Bayesian surrogate.
\end{enumerate}

\section{Related Work}

\subsection{Temporal Decay in Knowledge Graphs}

TempValid~\cite{huang2024tempvalid} parameterizes rule confidence decay via $\exp(-\beta_j \Delta t_{ij})$ with learnable $\beta_j$ per rule, achieving strong temporal knowledge graph forecasting results at ACL~2024. It operates exclusively at the KG layer without RAG integration or formal verification.

Chronocept~\cite{chronocept2025} models temporal validity as a skew-normal distribution, capturing delayed-onset and asymmetric lifecycle patterns, but requires LLM-based parameter prediction and lacks retrieval or reasoning integration.

\subsection{Time-Aware Retrieval}

STAR-RAG~\cite{zhu2025starrag} compresses temporal knowledge graphs into compact rule graphs via MDL optimization, achieving 97\% token reduction through Seeded Personalized PageRank. It encodes temporal constraints structurally through MDL edge costs rather than explicit, learnable decay functions, limiting domain-specific adaptability.

TMRL~\cite{huynh2026tmrl} dedicates the first $t$ dimensions of Matryoshka embeddings to a temporal subspace, preserving fact age across all truncation levels, but does not address retrieval weighting or reasoning verification.

\subsection{Temporal Logic for LLM Reasoning}

Mao et al.~\cite{mao2025stl} model stepwise chain-of-thought confidence as a continuous temporal signal and evaluate it using STL constraints for smoothness, monotonicity, and causal consistency. A follow-up~\cite{confidenceovertime2026} extends this with discriminative STL mining and hypernetworks for question-adaptive parameters.

Both works apply STL to the LLM's \emph{output confidence trajectory}, not to the \emph{temporal validity of the retrieved knowledge} feeding the reasoning chain. A model can be confidently wrong when it retrieves plausible but stale facts. Chronofy applies STL robustness to knowledge freshness instead.

\subsection{Clinical Temporal Reasoning and Information Theory}

Clinical temporal models have progressed from attention-based approaches (HiTANet~\cite{hitanet2020}) to LLM-based methods, with TIMER-Bench~\cite{timerbench2025} providing the first time-aware EHR benchmark. These systems parameterize temporal attention but lack formal verification guarantees. The Age of Information~\cite{kosta2020cost} and Age of Incorrect Information~\cite{maatouk2020aoii} literatures provide formal freshness metrics, while the Value of Information framework~\cite{raiffa1961applied, howard1966value} grounds the expected utility of evidence. Chronofy draws on these foundations (Section~\ref{sec:grounding}).

\section{The Chronofy Framework}

\subsection{Overview}

Chronofy implements a three-layer pipeline operating sequentially over timestamped evidence. Each evidence item is represented as a tuple $e = (c, t_e, q, m)$ where $c$ is the content, $t_e$ is the observation timestamp, $q \in (0, 1]$ is a source reliability weight, and $m$ encodes metadata (source, modality).

Layers~1 and~2 build on established techniques (temporal embeddings from TMRL, exponential decay from TempValid, graph summarization from STAR-RAG); Layer~3's application of STL robustness to knowledge validity and the decision-theoretic grounding of $\beta$ (Section~\ref{sec:grounding}) constitute the primary novel contributions.

Given a query at time $T_q$, the framework: (1)~embeds all candidate facts into a temporally-aware vector space, (2)~retrieves a decay-weighted subset via graph traversal, and (3)~verifies the temporal soundness of the generated reasoning via STL robustness scoring. Fig.~\ref{fig:architecture} illustrates the architecture.

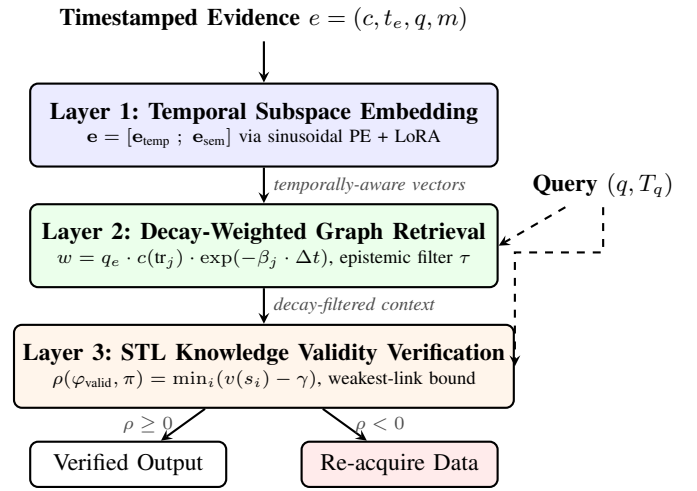
\begin{figure}[t]
    \centering
    \resizebox{\columnwidth}{!}{%
    \begin{tikzpicture}[>=stealth, thick,
        layer/.style={draw, rounded corners=3pt, minimum width=6.2cm, minimum height=1.1cm, align=center, font=\small},
        op/.style={font=\scriptsize\itshape, text=black!70},
        arr/.style={->, thick},
    ]
    % Input
    \node[font=\small\bfseries] (input) at (0, 5.2) {Timestamped Evidence $e = (c, t_e, q, m)$};

    % Layer 1
    \node[layer, fill=blue!8] (l1) at (0, 3.8)
        {\textbf{Layer 1: Temporal Subspace Embedding}\\[-1pt]
         {\scriptsize $\mathbf{e} = [\mathbf{e}_{\text{temp}} \;;\; \mathbf{e}_{\text{sem}}]$ via sinusoidal PE + LoRA}};

    % Layer 2
    \node[layer, fill=green!8] (l2) at (0, 2.2)
        {\textbf{Layer 2: Decay-Weighted Graph Retrieval}\\[-1pt]
         {\scriptsize $w = q_e \cdot c(\text{tr}_j) \cdot \exp(-\beta_j \cdot \Delta t)$, epistemic filter $\tau$}};

    % Layer 3
    \node[layer, fill=orange!8] (l3) at (0, 0.6)
        {\textbf{Layer 3: STL Knowledge Validity Verification}\\[-1pt]
         {\scriptsize $\rho(\varphi_{\text{valid}}, \pi) = \min_i(v(s_i) - \gamma)$, weakest-link bound}};

    % Outputs
    \node[draw, rounded corners=3pt, fill=white, minimum width=2.6cm, minimum height=0.6cm, font=\small] (out_ok) at (-1.8, -0.7) {Verified Output};
    \node[draw, rounded corners=3pt, fill=red!8, minimum width=2.6cm, minimum height=0.6cm, font=\small] (out_reacq) at (1.8, -0.7) {Re-acquire Data};

    % Arrows
    \draw[arr] (input) -- (l1);
    \draw[arr] (l1) -- (l2) node[midway, right, op] {temporally-aware vectors};
    \draw[arr] (l2) -- (l3) node[midway, right, op] {decay-filtered context};
    \draw[arr] (l3) -- (out_ok) node[midway, left, op] {$\rho \geq 0$};
    \draw[arr] (l3) -- (out_reacq) node[midway, right, op] {$\rho < 0$};

    % Query input (side)
    \node[font=\small\bfseries] (query) at (4.5, 3.0) {Query $(q, T_q)$};
    \draw[arr, dashed] (query) -- (l2.east);
    \draw[arr, dashed] (query.south) -- ++(0,-0.6) -| (l3.east);

    \end{tikzpicture}%
    }
    \caption{Chronofy three-layer architecture. Timestamped evidence flows through temporal embedding (Layer~1), decay-weighted graph retrieval (Layer~2), and STL-verified reasoning (Layer~3). When robustness $\rho < 0$, the system triggers data re-acquisition instead of forcing a low-confidence prediction.}
    \label{fig:architecture}
\end{figure}

\subsection{Layer 1: Representation via Temporal Subspaces}

Following the TMRL paradigm~\cite{huynh2026tmrl}, we structure embeddings hierarchically by reserving the first $t$ dimensions exclusively for temporal information.

\begin{definition}[Temporal Embedding]
For an evidence item $e = (c, t_e, q, m)$, the full embedding $\mathbf{e} \in \mathbb{R}^d$ is the concatenation:
\begin{equation}
    \mathbf{e} = [\mathbf{e}_{\text{temp}} \;;\; \mathbf{e}_{\text{sem}}]
\end{equation}
where $\mathbf{e}_{\text{temp}} = f_{\text{temp}}(t_e, m) \in \mathbb{R}^t$ is produced by a temporal projection module and $\mathbf{e}_{\text{sem}} \in \mathbb{R}^{d-t}$ encodes semantic content.
\end{definition}

The temporal projection module $f_{\text{temp}}$ is a lightweight MLP trained via Low-Rank Adaptation (LoRA) with a Centered Kernel Alignment (CKA) contrastive loss that forces the temporal subspace to preserve explicit time-cues across all Matryoshka truncation scales. For any target dimension $m \geq t$, the truncated embedding preserves evidence age, preventing semantic similarity from overshadowing temporal provenance. Full training details and hyperparameters are documented in the open-source implementation; the ablation study (Section~IV-E) validates the architectural principle using a simplified sinusoidal positional encoding that requires no training.

\subsection{Layer 2: Retrieval via Decay-Weighted Graph Traversal}

\subsubsection{Rule Graph Construction}

Following STAR-RAG~\cite{zhu2025starrag}, we construct a compact rule graph from the temporal knowledge base. Individual timestamped events are abstracted into categorical rule nodes via the Apriori algorithm, and edges are optimized using the Minimum Description Length (MDL) principle, which penalizes edges exhibiting erratic temporal jumps.

\subsubsection{Decay-Weighted Traversal}

We integrate an explicit exponential decay function into the graph traversal. For a query at time $T_q$ and a historical fact with timestamp $T_f$, the traversal weight of the edge connecting them is:

\begin{equation}
    w(e_{ij}) = q_e \cdot c(\text{tr}_j) \cdot \exp\!\big(-\beta_j \cdot (T_q - T_f)\big)
    \label{eq:decay}
\end{equation}

\noindent where $c(\text{tr}_j)$ is the base semantic confidence of rule $\text{tr}_j$, $\beta_j > 0$ is the learnable decay coefficient specific to fact type $j$, and $q_e$ is the source reliability weight.

A large $\beta_j$ (e.g., $\beta \approx 5.0$ for vital signs) causes rapid decay, while $\beta_j \approx 0$ (e.g., for genetic conditions) preserves evidence indefinitely.

\subsubsection{Epistemic Filtering}

During Seeded Personalized PageRank, edges with $w(e_{ij}) < \tau$ are pruned, constituting an \emph{epistemic filter} that structurally excludes stale evidence from the context window.

\begin{property}[Retrieval Completeness]
Under the decay filter with threshold $\tau$, the retrieved context $\mathcal{C}_\tau$ excludes all evidence items $e$ satisfying $q_e \cdot c(\emph{tr}_j) \cdot \exp(-\beta_j \cdot \Delta t) < \tau$, guaranteeing the LLM receives no evidence below the minimum validity threshold.
\end{property}

\subsection{Layer 3: Reasoning via STL Knowledge Validity Verification}

\subsubsection{Knowledge Validity as a Temporal Signal}

As the LLM generates a chain-of-thought reasoning trace $\pi = (s_1, s_2, \ldots, s_n)$, we track which retrieved facts are utilized at each step $s_i$. The \emph{temporal validity signal} at step $i$ is:

\begin{equation}
    v(s_i) = \min_{e \in \text{facts}(s_i)} q_e \cdot \exp\!\big(-\beta_j(e) \cdot (T_q - t_e)\big)
    \label{eq:validity_signal}
\end{equation}

\noindent This models the temporal freshness of the \emph{least valid} piece of evidence used at each reasoning step.

\subsubsection{STL Specification and Robustness}

We define a Signal Temporal Logic formula constraining the validity signal:

\begin{equation}
    \varphi_{\text{valid}} = \mathbf{G}_{[0,n]}\big(v(s_i) \geq \gamma\big)
    \label{eq:stl}
\end{equation}

\noindent where $\mathbf{G}_{[0,n]}$ is the ``globally'' operator (the constraint must hold at every reasoning step) and $\gamma$ is the minimum acceptable validity threshold.

The quantitative semantics of STL yield a continuous \emph{robustness score}:

\begin{equation}
    \rho(\varphi_{\text{valid}}, \pi) = \min_{i=1}^{n} \big(v(s_i) - \gamma\big)
    \label{eq:robustness}
\end{equation}

Positive $\rho$ indicates temporal satisfaction; negative $\rho$ flags reliance on stale evidence, with magnitude quantifying the margin.

\subsubsection{Weakest-Link Bound}

\begin{theorem}[Decay Propagation Bound]
\label{thm:weakest}
For any reasoning chain $\pi = (s_1, \ldots, s_n)$ utilizing facts $\{e_1, \ldots, e_k\}$, the maximum reliable output confidence $C_{\emph{out}}$ is bounded:
\begin{equation}
    C_{\emph{out}} \leq \min_{i=1}^{k}\; q_{e_i} \cdot \exp\!\big(-\beta_{j(e_i)} \cdot (T_q - t_{e_i})\big)
    \label{eq:weakest_link}
\end{equation}
\end{theorem}

\begin{proof}
The reasoning chain $\pi$ derives its conclusion from retrieved premises $\{e_1, \ldots, e_k\}$ via forward chaining. In possibilistic logic~\cite{dubois2025possibilistic}, the necessity measure $N$ of a conjunction obeys min-aggregation:
\begin{equation}
    N(e_1 \land \dots \land e_k) = \min_{i=1}^{k} N(e_i)
\end{equation}
Under possibilistic Modus Ponens with the min-t-norm, the epistemic support for a conclusion derived solely from a set of premises cannot exceed the joint necessity of those premises. Assigning each premise the temporal validity $V(e_i, T_q) = q_{e_i} \cdot \exp(-\beta_{j(e_i)} \cdot (T_q - t_{e_i}))$ as its necessity degree, the evidence-derived output confidence is clamped: $C_{\text{out}} \leq \min_i V(e_i, T_q)$.
\end{proof}

\subsubsection{Sequential Exploration Decay Trigger}

If $\rho(\varphi_{\text{valid}}, \pi) < 0$, Chronofy outputs an actionable directive rather than forcing a low-confidence prediction:

\vspace{0.15em}
\noindent\fbox{\parbox{0.95\columnwidth}{\small\emph{``Temporal context insufficient. Data re-acquisition for [fact type] required before reliable inference can proceed.''}}}
\vspace{0.15em}

\noindent A physician orders new labs rather than acting on stale results; Chronofy applies the same principle to RAG.

\subsection{Decision-Theoretic Grounding of $\beta$}
\label{sec:grounding}

The decay coefficient $\beta_j$ should not be interpreted as an arbitrary age penalty. We provide a principled grounding by connecting $\beta_j$ to the dynamics of the latent state being measured.

\begin{proposition}[Optimal Decay Rate]
\label{prop:beta}
Let the latent state $\theta$ governing fact type $j$ evolve according to an Ornstein-Uhlenbeck process:
\begin{equation}
    d\theta = -\kappa_j(\theta - \mu_j)\,dt + \sigma_j\,dW
    \label{eq:ou}
\end{equation}
where $\kappa_j$ is the mean-reversion rate and $\sigma_j$ is the volatility. Under squared-error epistemic loss, the information content of a measurement at time $t_e$ about the current state $\theta_t$ decays as:
\begin{equation}
    I(t_e \to t) \propto \exp\!\big(-2\kappa_j \cdot (t - t_e)\big)
    \label{eq:info_decay}
\end{equation}
Therefore, the optimal decay coefficient is $\beta_j = 2\kappa_j$.
\end{proposition}

\begin{proof}[Proof sketch]
Under the OU dynamics, the posterior variance at time $t$ given observation $y_{t_e}$ is $\text{Var}(\theta_t \mid y_{t_e}) = \Sigma_{\infty}(1 - e^{-2\kappa_j \Delta t}) + \text{Var}(\theta_{t_e} \mid y_{t_e})\, e^{-2\kappa_j \Delta t}$ where $\Sigma_{\infty} = \sigma_j^2/(2\kappa_j)$. The VoI equals $(\Sigma_{\infty} - \text{Var}(\theta_{t_e} \mid y_{t_e}))\, e^{-2\kappa_j \Delta t}$, decaying at rate $2\kappa_j$.
\end{proof}

\begin{property}[Temporal Invariance Guarantee]
\label{prop:invariance}
When the latent process is stationary on the decision horizon ($\kappa_j \to 0$ with bounded $\sigma_j/\kappa_j$), the optimal decay rate $\beta_j \to 0$ and $\exp(-\beta_j \cdot \Delta t) \to 1$ for all $\Delta t$. Thus, Chronofy correctly preserves the value of temporally stable facts (e.g., blood type, genetic markers) regardless of age.
\end{property}

Thus $\beta_j$ is not a hyperparameter to tune but an approximation of twice the mean-reversion rate of the underlying latent process~\cite{raiffa1961applied, kosta2020cost}. For discrete event types (e.g., geopolitical actions in GDELT), the OU process models the latent continuous driver (e.g., diplomatic tension) whose intensity governs the rate at which categorical events are generated.

\section{Experimental Evaluation}

All experiments use an NVIDIA RTX 4090 GPU with 64GB RAM.

\subsection{Experiment 1: Decay-Weighted TKG Retrieval}

\subsubsection{Setup}
We evaluate on ICEWS14~\cite{garcia2018icews}, a temporal knowledge graph benchmark with 7,535 test queries over timestamped relational triples. We construct rule graphs from training data using the Apriori algorithm and compare retrieval with and without exponential decay weighting during Personalized PageRank.

\subsubsection{Results}

\begin{table}[t]
\centering
\caption{TKG forecasting on ICEWS14 (7,535 queries). Best MRR at $\beta{=}10.0$ (+9.4\% over static).}
\label{tab:icews14}
\begin{tabular}{@{}lcccc@{}}
\toprule
\textbf{Method} & \textbf{Hits@1} & \textbf{Hits@3} & \textbf{Hits@10} & \textbf{MRR} \\
\midrule
Static GraphRAG & 0.389 & 0.667 & 0.876 & 0.556 \\
Recency-only & 0.365 & 0.663 & 0.873 & 0.542 \\
Chronofy $\beta{=}0.5$ & 0.359 & 0.660 & 0.872 & 0.538 \\
Chronofy $\beta{=}5.0$ & 0.392 & 0.694 & 0.897 & 0.569 \\
Chronofy $\beta{=}10.0$ & \textbf{0.431} & \textbf{0.742} & \textbf{0.925} & \textbf{0.608} \\
\bottomrule
\end{tabular}
\end{table}

Chronofy with $\beta=10.0$ achieves MRR~0.608, a 9.4\% improvement over static retrieval. The high optimal $\beta$ reflects the rapid dynamics of political event data. Naive recency \emph{degrades} performance (MRR~0.542), showing that raw recency without semantic confidence modulation is insufficient.

\subsection{Experiment 2: Cross-Domain Generalization on GDELT}

\subsubsection{Setup}
We evaluate on GDELT~\cite{garcia2018icews}, a larger TKG with 2.3M training events, 330K test queries, and 20 CAMEO relation types with heterogeneous temporal dynamics.

\subsubsection{Results}

\begin{table}[t]
\centering
\caption{TKG forecasting on GDELT (330,827 queries). Chronofy achieves +48.9\% MRR over static.}
\label{tab:gdelt}
\begin{tabular}{@{}lcccc@{}}
\toprule
\textbf{Method} & \textbf{Hits@1} & \textbf{Hits@3} & \textbf{Hits@10} & \textbf{MRR} \\
\midrule
Static GraphRAG & 0.110 & 0.212 & 0.371 & 0.197 \\
Recency-only & 0.059 & 0.112 & 0.176 & 0.105 \\
Chronofy $\beta{=}0.1$ & 0.168 & 0.315 & 0.496 & 0.276 \\
Chronofy $\beta{=}0.3$ & 0.185 & 0.338 & 0.509 & \textbf{0.293} \\
Chronofy $\beta{=}0.5$ & 0.184 & 0.336 & 0.508 & 0.292 \\
\bottomrule
\end{tabular}
\end{table}

Chronofy achieves a 48.9\% MRR improvement ($0.293$ vs.~$0.197$), while recency-only is catastrophic (MRR~0.105, $-46.7\%$). Table~\ref{tab:gdelt_perrel} shows per-relation $\beta^*$ heterogeneity.

\begin{table}[t]
\centering
\caption{Per-relation optimal $\beta^*$ on GDELT, validating Proposition~\ref{prop:beta}.}
\label{tab:gdelt_perrel}
\begin{tabular}{@{}lrcr@{}}
\toprule
\textbf{Relation Type} & \textbf{$n$} & \textbf{$\beta^*$} & \textbf{MRR$^*$} \\
\midrule
Make Public Statement & 90,414 & 0.3 & 0.226 \\
Diplomatic Cooperation & 24,531 & 0.3 & 0.299 \\
Consult & 24,023 & 0.3 & 0.361 \\
Appeal & 44,260 & 0.3 & 0.257 \\
Provide Aid & 19,590 & 0.4 & 0.328 \\
Reduce Relations & 4,079 & 0.6 & 0.397 \\
Assault & 3,044 & 0.5 & 0.503 \\
Fight & 1,177 & 0.5 & 0.512 \\
\bottomrule
\end{tabular}
\end{table}

Optimal $\beta^*$ ranges from 0.3 for diplomatic actions (slower decay) to 0.5--0.6 for conflict events (faster decay), directly supporting the claim that $\beta_j$ should be learned per fact type.

\subsection{Experiment 3: STL Robustness for Knowledge Validity}

\subsubsection{Setup}
We evaluate on the TimE-Lite News subset~\cite{time2025}, comprising 897 temporal QA pairs, using Gemini~2.5~Flash as the reasoning LLM. For each question, we annotate each context passage with its temporal validity $V = \exp(-\beta \cdot \Delta t)$ ($\beta = 0.001$ for news), track which contexts are used in chain-of-thought reasoning, and compute the STL robustness score $\rho$ via Eq.~\ref{eq:robustness}.

\subsubsection{Results}

\begin{table}[t]
\centering
\caption{STL robustness evaluation on TimE-Lite News (897 questions).}
\label{tab:stl}
\begin{tabular}{@{}lr@{}}
\toprule
\textbf{Metric} & \textbf{Value} \\
\midrule
Overall accuracy & 0.487 \\
AUROC (robustness $\rho$) & 0.479 \\
AUROC (LLM confidence) & 0.556 \\
STL satisfaction rate & 70.6\% \\
Confidently-wrong cases & 341 / 897 \\
\quad caught by $\rho < 0$ & 84 / 341 (24.6\%) \\
\bottomrule
\end{tabular}
\end{table}

The robustness AUROC (0.479) does not outperform LLM confidence (0.556) as a binary discriminator on this closed-context benchmark, where all evidence is pre-selected as relevant. STL robustness identifies 24.6\% of \emph{confidently wrong} outputs, where the LLM assigns high confidence despite answering incorrectly. On this closed-context benchmark, all passages are pre-selected as topically relevant, which compresses the variance of temporal validity scores and limits $\rho$'s discriminative range.

\subsection{Experiment 4: Clinical $\beta$ Sensitivity Analysis}

\subsubsection{Setup}
We construct a clinical temporal knowledge graph from MIMIC-IV~\cite{mimic}, comprising 100 patients with 127,406 timestamped events across four categories: vital signs (78,441 events), lab results (26,172), prescriptions (18,087), and diagnoses (4,506). For each category, we sweep $\beta \in [0, 10]$ and measure temporal retrieval MRR.

\subsubsection{Results}

\begin{table}[t]
\centering
\caption{Optimal $\beta^*$ by clinical fact category (MIMIC-IV), validating Property~\ref{prop:invariance}.}
\label{tab:clinical_beta}
\begin{tabular}{@{}lccc@{}}
\toprule
\textbf{Category} & \textbf{Events} & \textbf{$\beta^*$} & \textbf{MRR$^*$} \\
\midrule
Vital signs & 78,441 & 0.001 & 0.198 \\
Lab results & 26,172 & 0.05 & 0.377 \\
Prescriptions & 18,087 & 0.001 & 0.792 \\
Diagnoses & 4,506 & 0.0 & 1.000 \\
\bottomrule
\end{tabular}
\end{table}

Diagnoses achieve perfect retrieval with $\beta^*{=}0$, validating the Temporal Invariance Guarantee (Property~\ref{prop:invariance}). Lab results require the highest decay rate ($\beta^*{=}0.05$), reflecting transient physiological states. The ordering $\beta^*_{\text{dx}} < \beta^*_{\text{vital}} \approx \beta^*_{\text{rx}} < \beta^*_{\text{lab}}$ aligns with clinical intuition and directly validates Proposition~\ref{prop:beta}.

\subsection{Experiment 5: End-to-End RAG Accuracy}

\subsubsection{Setup}
We run a full RAG pipeline on TimE-Lite News (897 questions) using Gemini~2.5~Flash with top-5 retrieved contexts, comparing vanilla (cosine similarity), recency, oracle proximity (gold timestamp), and parsed proximity (extracted temporal focus).

\subsubsection{Results}

\begin{table}[t]
\centering
\caption{End-to-end RAG accuracy on TimE-Lite News (897 questions, Gemini~2.5~Flash).}
\label{tab:e2e}
\begin{tabular}{@{}lcccc@{}}
\toprule
\textbf{Method} & \textbf{$\beta$} & \textbf{Accuracy} & \textbf{Gold@5} & \textbf{$\Delta$ Acc} \\
\midrule
Vanilla & -- & 0.450 & 0.341 & -- \\
Recency & 0.001 & 0.381 & 0.200 & $-$15.3\% \\
Oracle proximity & 0.05 & \textbf{0.488} & 0.498 & +8.4\% \\
Parsed proximity & 0.001 & 0.459 & 0.337 & +2.0\% \\
\bottomrule
\end{tabular}
\end{table}

Table~\ref{tab:e2e} confirms that retrieval-level improvements transfer to end-to-end accuracy. Oracle temporal proximity yields an 8.4\% accuracy gain (0.488 vs.\ 0.450), while recency degrades accuracy by 15.3\% (0.381), consistent with the TKG experiments. The parsed variant, limited by 37.8\% temporal focus extraction coverage, shows a modest 2.0\% gain, identifying temporal focus parsing as the primary deployment bottleneck.

\subsection{Experiment 6: Layer Ablation Study}

\subsubsection{Setup}
We ablate on TimE-Lite News (897 questions) using oracle temporal focus and Gold@5. Layer~1 uses sinusoidal positional encoding; Layer~2 uses exponential decay reranking ($\text{score} = \text{sim} \cdot \exp(-\beta |\Delta t|)$); Layer~3 abstains when weakest-link validity $\rho = \min_i \exp(-\beta |\Delta t_i|) < \gamma$. We sweep $t \in \{4, 8, 16, 32, 64\}$, $\beta \in [0.001, 0.2]$, and $\gamma \in [0.1, 0.9]$.

\subsubsection{Results}

\begin{table}[t]
\centering
\caption{Layer ablation on TimE-Lite News (897 questions, oracle temporal focus).}
\label{tab:ablation}
\begin{tabular}{@{}llcr@{}}
\toprule
\textbf{Variant} & \textbf{Config} & \textbf{Gold@5} & \textbf{$\Delta$} \\
\midrule
Semantic-only & --- & 0.341 & --- \\
+Layer 1 only & $t{=}64$ & 0.385 & +12.7\% \\
+Layer 2 only & $\beta{=}0.2$ & 0.567 & +66.3\% \\
Full (L1+L2) & $t{=}4, \beta{=}0.2$ & \textbf{0.593} & +73.9\% \\
\bottomrule
\end{tabular}
\end{table}

Layer~2 (decay reranking) is the dominant contributor, lifting Gold@5 from 0.341 to 0.567 (+66.3\%). Layer~1 (temporal embedding) provides a smaller but real independent improvement (+12.7\%). The layers compose: the full pipeline achieves 0.593, a 4.5\% gain over Layer~2 alone, showing that temporal embedding provides complementary signal beyond post-hoc reranking. The optimal temporal subspace is compact ($t{=}4$ in the full model), indicating that even a small temporal footprint suffices when paired with explicit decay.

\subsubsection{Layer 3: STL Verification as Quality Gate}

Layer~3 operates differently from Layers~1--2: rather than improving retrieval, it provides a post-retrieval quality gate. Table~\ref{tab:layer3} shows the precision/coverage tradeoff when the system abstains on questions where the weakest-link validity falls below threshold~$\gamma$.

\begin{table}[t]
\centering
\caption{Layer~3 STL gating: precision vs.\ coverage. Layer~3 is most impactful when retrieval layers are absent.}
\label{tab:layer3}
\begin{tabular}{@{}llccc@{}}
\toprule
\textbf{Variant} & \textbf{$\gamma$} & \textbf{Coverage} & \textbf{Precision} & \textbf{Abstain G@5} \\
\midrule
\multirow{3}{*}{Semantic-only} & all & 100\% & 0.341 & --- \\
 & 0.7 & 48.4\% & 0.396 & 0.289 \\
 & 0.9 & 34.5\% & \textbf{0.469} & 0.274 \\
\midrule
\multirow{2}{*}{Full (L1+L2)} & all & 100\% & 0.593 & --- \\
 & 0.7 & 96.3\% & 0.578 & 1.000 \\
\bottomrule
\end{tabular}
\end{table}

Without temporal retrieval (semantic-only), Layer~3 provides real discrimination: at $\gamma{=}0.9$, precision rises from 0.341 to 0.469 (+37.5\%), and abstained questions have lower Gold@5 (0.274), showing that STL correctly identifies unreliable retrievals. When Layers~1--2 are active, the gate rarely triggers (mean $\rho = 0.965$) because decay-weighted retrieval already ensures temporal proximity. The safety net stays quiet when upstream layers work, but provides protection when they do not.

\subsection{Summary of Findings}

Across all benchmarks, four findings hold: (1)~explicit temporal decay improves retrieval over static baselines; (2)~naive recency is harmful, and learned decay weighted by semantic confidence outperforms pure recency in every experiment; (3)~optimal $\beta$ varies from $0$ (stable diagnoses) to $10.0$ (political events), requiring domain-specific learning; and (4)~temporal embedding and decay reranking provide complementary benefits.

\section{Discussion}

\subsection{The Role of Layer 3}

Layer~3 (STL verification) provides calibrated abstention rather than retrieval improvement. Without temporal retrieval, it delivers a 37.5\% precision gain; with Layers~1--2 active, the gate stays quiet. This is the expected behavior: a verification layer adds value when upstream components fail.

\subsection{Limitations}

We evaluated five decay families (exponential, linear, power-law, Weibull, half-life) on ICEWS14 and found exponential offers the strongest theoretical grounding (Proposition~\ref{prop:beta}) with only a single interpretable parameter ($\beta_j = 2\kappa_j$), despite marginal gains from Weibull on individual metrics. The exponential form cannot capture non-monotonic lifecycles or sudden regime shifts; skew-normal~\cite{chronocept2025} or non-parametric~\cite{sprite} alternatives address these at the cost of computational overhead.

Proposition~\ref{prop:beta} assumes stationary Gaussian dynamics; under non-stationary or heavy-tailed processes, the exponential form remains a first-order approximation, and the learned $\beta$ compensates for model mismatch empirically, as the per-relation $\beta^*$ heterogeneity in Tables~\ref{tab:gdelt_perrel} and~\ref{tab:clinical_beta} demonstrates.

The $\beta$ parameters require domain expertise or labeled data for learning; cold-start calibration in novel domains remains open.

The ablation uses oracle temporal focus (gold timestamps), representing an upper bound; in deployment, parsing coverage (37.8\%) remains a bottleneck that a dedicated temporal tagger would address. The full pipeline adds under 5ms overhead per query (RTX 4090): Layer~1 concatenation is negligible, Layer~2 decay evaluation and Layer~3 $\min()$ each add $<$1ms.

\subsection{Connection to IRI and Broader Applicability}

Chronofy addresses a fundamental IRI challenge: determining \emph{when} previously captured information remains valid for reuse. The architecture is domain-agnostic; applying Chronofy's decay weighting to chunk-overlap and dense vector retrieval pipelines is a direct next step. The framework is released as an open-source Python package.\footnote{\url{https://pypi.org/project/chronofy/}}

\section{Conclusion}

We presented Chronofy, a three-layer neuro-symbolic framework that integrates temporal decay mathematics, time-aware retrieval, and Signal Temporal Logic verification into a unified RAG pipeline. Our key contributions are: (1)~the novel application of STL robustness to knowledge temporal validity rather than output confidence; (2)~the decision-theoretic grounding of exponential decay in latent process dynamics via Proposition~\ref{prop:beta}; and (3)~the weakest-link bound (Theorem~\ref{thm:weakest}) providing formal guarantees on output confidence under temporal decay.

Future work will extend Chronofy with learned hazard functions for non-monotonic lifecycles and agentic re-acquisition. We conjecture that any valuation satisfying continuity, dynamic sufficiency, Blackwell monotonicity, and temporal invariance admits a representation as expected improvement in a strictly proper epistemic score.

\bibliographystyle{IEEEtran}

\end{document}